\documentclass{article}

\usepackage{microtype}
\usepackage{graphicx}
\usepackage{subfigure}
\usepackage{booktabs} % for professional tables

\usepackage{hyperref}

\PassOptionsToPackage{numbers}{natbib}
\usepackage[preprint]{neurips_2020}
\usepackage{bbm}
\usepackage{url}            % simple URL typesetting
\usepackage{lipsum}
\usepackage{amsthm}
\usepackage{amsmath}
\usepackage{appendix}
\usepackage{hyperref}       % hyperlinks
\usepackage{booktabs}       % professional-quality tables
\usepackage{amsfonts}       % blackboard math symbols
\usepackage{nicefrac}       % compact symbols for 1/2, etc.
\usepackage{todonotes}
\presetkeys%
    {todonotes}%
    {inline,backgroundcolor=yellow}{}
\usepackage{graphicx}
\usepackage{microtype}      % microtypography
\usepackage[T1]{fontenc}    % use 8-bit T1 fonts
\usepackage[utf8]{inputenc} % allow utf-8 input

\newtheorem{lemma}{Lemma}[section]
\newtheorem{proposition}{Proposition}[section]
\newtheorem{theorem}{Theorem}[section]

\newcommand\figref{Figure~\ref}
\newcommand\expx[1]{\mathbb{E}_\mathbf{x} \left[ {#1} \right]}
\newcommand\expxt[1]{\mathbb{E}_{\tilde{\mathbf{x}}} \left[ {#1} \right]}

\newcommand\bw{\mathbf{w}}
\newcommand\bb{\mathbf{b}}

\newcommand\bwfrac{\frac{b}{\|\bw\|}}

\newcommand\ba{\mathbf{a}}
\newcommand\bx{{\mathbf{x}}}

\newcommand\bm{{\mathbf{m}}}
\newcommand\R{{\mathbb{R}}}
\renewcommand\L{{\mathcal{L}}}

\renewcommand\eqref[1]{Eq. (\ref{#1})}

\allowdisplaybreaks[1]

\title{The effect of Target Normalization and Momentum on Dying ReLU}

\author{%
  Isac Arnekvist \\
  Division of Robotics, Perception and Learning\\
  Royal Institute of Technology (KTH)\\
  Stockholm \\
  \texttt{isacar@kth.se} \\
  \And
  J. Frederico Carvalho \\
  Division of Robotics, Perception and Learning\\
  Royal Institute of Technology (KTH)\\
  Stockholm \\
  \texttt{jfpbdc@kth.se} \\
  \And
  Danica Kragic \\
  Division of Robotics, Perception and Learning\\
  Royal Institute of Technology (KTH)\\
  Stockholm \\
  \texttt{dani@kth.se} \\
  \And
  Johannes A. Stork \\
  Center for Applied Autonomous Sensor Systems\\
  Örebro University\\
  Örebro \\
  \texttt{johannesandreas.stork@oru.se} \\
}

\begin{document}

\maketitle

\begin{abstract}

Optimizing parameters with momentum, normalizing data values, and using
rectified linear units (ReLUs) are popular choices in neural network (NN)
regression. Although ReLUs are popular, they can collapse to a constant function
and ``die'', effectively removing their contribution from the model. While some
mitigations are known, the underlying reasons of ReLUs dying during optimization
are currently poorly understood. In this paper, we consider the effects of
target normalization and momentum on dying ReLUs. We find empirically that unit
variance targets are well motivated and that ReLUs die more easily, when target
variance approaches zero. To further investigate this matter, we analyze a
discrete-time linear autonomous system, and show theoretically how this relates
to a model with a single ReLU and how common properties can result in dying
ReLU. We also analyze the gradients of a single-ReLU model to identify saddle
points and regions corresponding to dying ReLU and how parameters evolve into
these regions when momentum is used. Finally, we show empirically that this
problem persist, and is aggravated, for deeper models including
residual networks.

\end{abstract}

% !TEX root = main.tex
\section{Introduction}\label{sec:introduction}

Gradient-based optimization enables learning of powerful deep NN
models \cite{DarganShaveta2019ASoD,rumelhart1986learning}. However, most learning algorithms remain
sensitive to learning rate, scale of data values, and the choice of activation
function---making deep NN models hard to train
\cite{srivastava2015training, du2019gradient}. Stochastic
gradient descent with momentum \cite{sutskever2013importance,adam}, normalizing data values to have
zero mean and unit variance \cite{lecun2012efficient}, and employing rectified linear units
(ReLUs) in NNs \cite{lecun2015deep, ramachandran2017searching,
nair2010rectified} have emerged as an empirically motivated and popular
practice. In this paper, we analyze a specific failure case of this practice,
referred to as ``dying'' ReLU.

The ReLU activation function, $y = \max \{x, 0\}$ is a popular choice of activation
function and has been shown to have superior
training performance in various domains \cite{glorot2011deep, sun2015deeply}.
ReLUs can sometimes be collapse to a constant (zero) function for a given set
of inputs. Such a ReLU is considered ``dead'' and does not contribute to a
learned model. ReLUs can be initialized dead \cite{lu2019dying} or die during
optimization, the latter being a major obstacle in training deep NNs
\cite{trottier2017parametric, agarap2018deep}. Once dead, gradients are zero
making recovery possible only if inputs change distribution.
Over time, large parts of a NN can end up dead which reduces model capacity.

Mitigations to dying ReLU include
modifying the ReLU to also activate for negative inputs
\cite{maas2013rectifier,clevert2016fast,he2015delving}, training procedures with
normalization steps \cite{ba2016layer, ioffe2015batch}, and initialization
methods \cite{lu2019dying}. While these approaches have some success in
practice, the underlying cause for ReLUs dying during optimization is, to our
knowledge, still not understood.

\begin{figure}[ht!]
    \centering
    \includegraphics[width=0.45\textwidth]{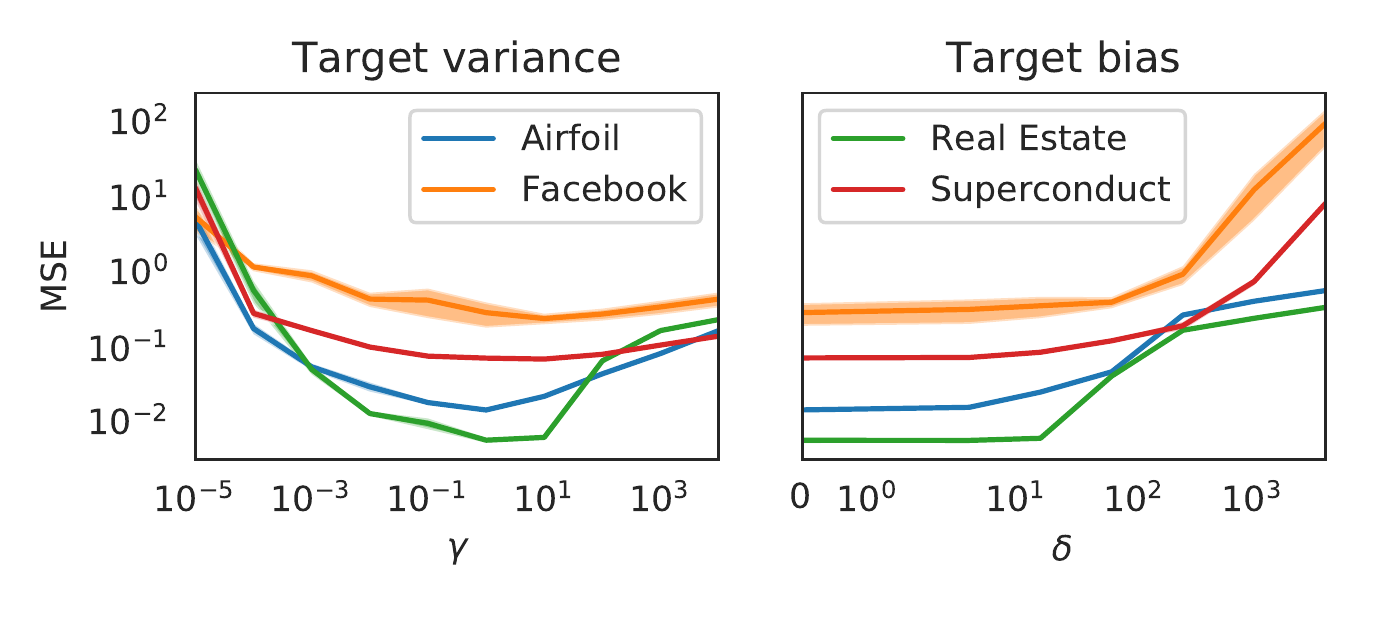}
    %\caption{
%Effect of target normalization with scale $\gamma$ and bias $\delta$ on training MSE (left: $\delta = 0$, right: $\gamma = 1$). Shaded region shows $\pm$ the standard error of mean. For $\gamma = 1$ and $\delta = 0$, which is the common choice of zero mean and unit variance, a ``sweet spot'' can be observed.}
    %\label{fig:scale_bias_rmse}
    \centering
    \includegraphics[width=0.45\textwidth]{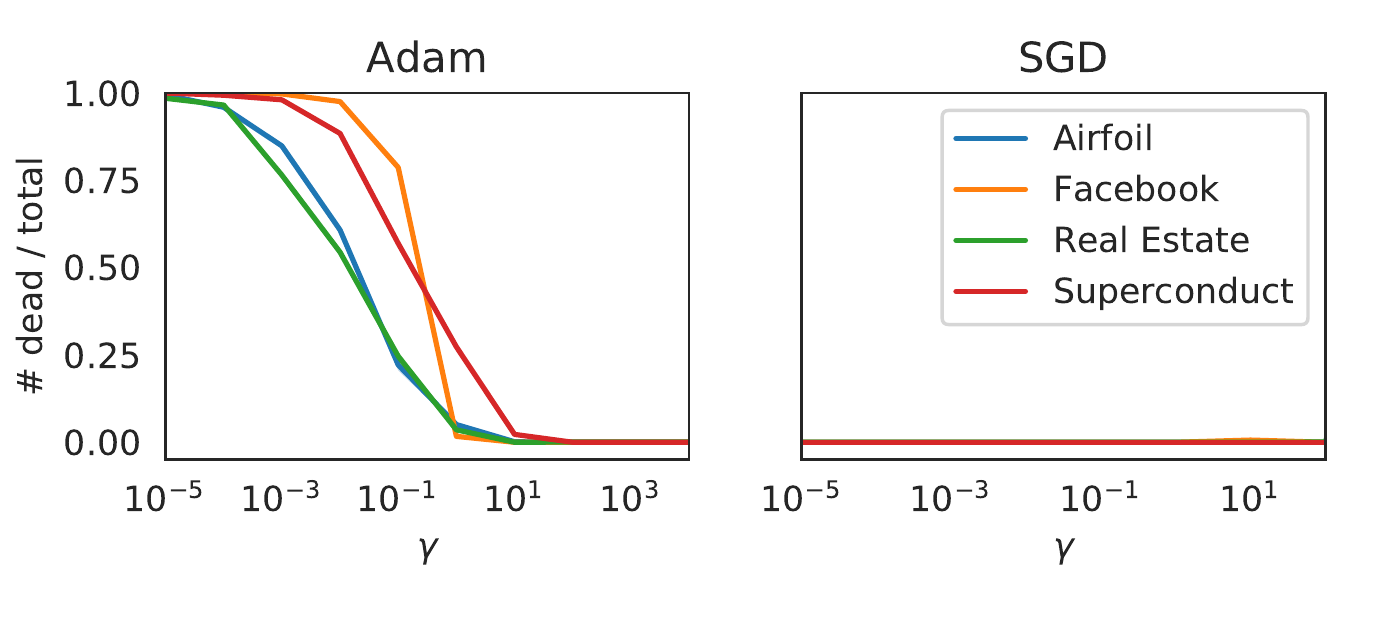}

    \caption{
  The performance of models, in terms of fit to the training data, degrades as
  the variance ($\gamma$) of target values deviate from $1$, or the mean
  ($\delta$) deviates from $0$. Here shown in terms of mean-squared error (MSE)
  on four different datasets \cite{uci}. Furthermore, decreasing $\gamma$ leads
  to more dying ReLUs only with the use of momentum optimization \cite{adam}.
  Details of this experiment can be found in Appendix \ref{regression_details}.}

    \label{fig:overview}
\end{figure}

In this paper, we analyze the observation illustrated in \figref{fig:overview}
that regression performance degrades with smaller target variances, and along with
momentum optimization leads to dead ReLU. Although target normalization is a
common pre-processing step, we believe a scientific understanding of \emph{why}
it is important is missing, especially with the connection to momentum
optimization.
For our theoretical results,
we first show that an affine approximator trained with gradient descent and momentum corresponds
to a discrete-time linear autonomous system. Introducing momentum into this system
results in complex eigenvalues and parameters that oscillate. We further
show that a single-ReLU model has two cones in parameter space; one for which
the properties of the linear system is shared, and one that corresponds to dead ReLU.

We derive analytic gradients for the single-ReLU model to further gain insight
and to identify critical points (i.e. global optima and saddle points) in
parameter space. By inspection of numerical examples, we also identify regions
where ReLUs tend to converge to the global optimum (without dying) and how these
regions change with momentum. Lastly, we show empirically that the problem of
dying ReLU is aggravated in deeper models, including residual neural networks.

% !TEX root = main.tex
\section{Related work}

% This section is only one page long. Maybe the reader can manage without subsection. It save quite some space.
%\subsection{Dying ReLU}

In a recent paper \cite{lu2019dying}, the authors identify dying ReLUs
as a cause of vanishing gradients. This is a fundamental
problem in NNs \cite{poole2016exponential, hanin2018neural}. In
general, this can be caused by ReLUs being initialized dead or dying during
optimization. Theoretical results about initialization and dead ReLU NNs
are presented by \citet{lu2019dying}. Growing NN depth towards infinity and
initializing parameters from symmetric distributions both lead to dead models.
However, asymmetric initialization can effectively prevent this outcome.
Empirical results about ReLUs dying during optimization are presented by
\citet{wu2018ans}. Similar to us, they observe a relationship between dying
ReLUs and the scale of target values. In contrast to us, they do not investigate
the underlying cause.

%This insight suggest that applying normalization methods to target values could be useful for addressing dying ReLUs.

%\citet{lu2019dying} showed that a neural network can be ``born'' dead, and that the probability of so happening goes to one as the depth of the network increases.

%Dying ReLU has been identified as an obstacle to training deep neural networks, and can even be a problem already after initialization, before optimization starts.

%% Moved to introduction
%Several novel activation functions have been proposed to alleviate the problem of dying ReLU \citep{trottier2017parametric,he2015delving,maas2013rectifier} but still ReLU remains an attractive alternative due to its simplicity, the resulting sparsity in activations \citep{glorot2011deep}, and its competitive performance \citep{ramachandran2017searching}.

%That said, ReLU will likely remain a common choice as activation function in application and research, but the connection between the variance of target values and dying ReLU is still not understood.

%\subsection{Data Normalization}

\paragraph{Normalization of layer input values.}

%Normalization of values pertaining to the training of neural networks can be divided into two categories: (1) Normalization of input values, here including inputs to the neural network originating from the dataset, but also inputs to subsequent layers in the network produced by previous layers. (2) Normalization of target values, i.e., the values from the dataset that we aim to predict.

The effects of input value distribution has been studied for a long time, e.g.
\cite{sola1997importance}. Inputs with zero mean have been shown to result in
gradients that more closely resemble the natural gradient, which speeds up
training \cite{raiko2012deep}. In fact, a range of strategies to
normalize layer input data exists \citep{ioffe2015batch,ba2016layer,ioffe2017self} along
with theoretical analysis of the problem \citep{santurkar2018does}.
Another studied area for maintaining statistics throughout the NN is
initialization of the parameters
\citep{glorot2010understanding,he2015delving,lu2019dying}. However, subsequent
optimization steps may change the parameters such that the desired input mean
and variance no longer is fulfilled.

\paragraph{Normalization of target values.}

When the training data are available before optimization, target normalization is trivially executed. More challenging is the case where training data are accessed incrementally, e.g. as in reinforcement learning or for very large training data. Here, normalization and scaling of target values are important for the learning outcome \cite{van2016learning, henderson2018deep, wu2018ans}.
For on-line regression and reinforcement learning, adaptive target normalization
improves results and removes the need of gradient clipping
\cite{van2016learning}. In reinforcement learning, scaling rewards by a positive
constant is crucial for learning performance, and is often equivalent to the
scaling of target values \cite{henderson2018deep}. Small reward scales have been
seen to increase the risk of dying ReLUs \cite{wu2018ans}. All of these works
motivate the use of target normalization empirically and a theoretical
understanding is still lacking. In this paper, we provide more insight into the
relationship between dying ReLUs and target normalization.

% !TEX root = main.tex
\section{Preliminaries}
\label{sec:preliminaries}

We consider regression of a target function $y \colon \mathbb{R}^L \to
\mathbb{R}$ from training data $\mathcal{D} = \{ (\mathbf{x}^{(i)}, y_*^{(i)})
\}_i$ consisting of pairs of inputs $\mathbf{x}^{(i)} \in \mathbb{R}^L$ and
target values $y_*^{(i)} = y(\mathbf{x}^{(i)})\in \mathbb{R}$. We analyze
different regression models $\hat y$, such as an affine transformation in Sec.
\ref{sec:affine} and a ReLU-activated model in Sec. \ref{sec:single_relu}, which are
both parameterized by a vector $\boldsymbol \theta$. Below, we provide definitions,
notations, and equalities needed for our analysis.

%\todo{
%1) The notation for the training data is a bit misleading because it suggests that $i$ is from the integers and the training set is maximally countable. This would be in conflict with the uncountable training set of the generalized conditions where the distribution is used. Technically this is not work since $i$ could be from the reals. I have thought about removing the indexing with $i$ and if it is too distracting it can be done by introducing differnt notations for target as function and target as value. One could also just write some words that say that $i$ is just a notation.}

\paragraph{Target normalization.}

Before regression, we transform target values $y_*^{(i)}$ according to
\begin{equation}
y^{(i)} = \gamma \frac{y_*^{(i)} - \hat \mu}{\hat\sigma} + \delta
,
\end{equation}
where $\hat \mu$ and $\hat \sigma$ are mean and standard deviation of the target
values from the training data. When the parameters of the transform are set to
scale $\gamma = 1$ and bias $\delta = 0$, new target values
$y^{(i)}$ correspond to $z$-normalization \cite{goldin1995similarity} with zero
mean and unit variance. In our analysis, we are interested in the effects of
$\gamma$ changing from $1$ to smaller values closer to $0$.

\paragraph{Target function.}

In Sec. \ref{sec:affine} we study regression of target functions of the form
\begin{equation}
y(\bx) = \boldsymbol \Gamma^\top \bx + \Delta
\label{eq:desired_response}
\end{equation}
where $\boldsymbol \Gamma \in \mathbb R^{L \times 1}$ and $\Delta \in \R$.

Similar to \citet{douglas2018relu}, we consider the case where inputs
in $\mathcal{D}$ are distributed as $\mathbf x \sim \mathcal{N}(\mathbf{0};
\mathbf{I})$. For any $\boldsymbol \Gamma$, we can find a unitary matrix $\mathbf{O} \in \mathbb{R}^{L \times
L}$ such that
$\mathbf{O}\boldsymbol \Gamma = (0, \dots, 0, \gamma)^\top$ and $\gamma = \|\boldsymbol \Gamma\|$.
From this follow the equalities
\begin{align}
\boldsymbol \Gamma^\top \mathbf x + \Delta
=
\boldsymbol \Gamma^\top \mathbf{O}^\top \mathbf{O} \mathbf x + \Delta
=
(0, \dots ,\gamma) \mathbf{O}\mathbf x + \Delta.
\end{align}
Since $\bx$ and $\mathbf{O}\bx$ are identically distributed due to our
assumption on $\bx$, we can
equivalently study the target function
\begin{equation}
y(\bx) = (0, \dots , 0, \gamma) \mathbf x + \Delta
\label{eq:target_special}
\end{equation}
and assume that $\boldsymbol \Gamma = (0, \dots, 0, \gamma)^\top$ for the remainder of this paper.

For Sec. \ref{sec:single_relu} we consider a ReLU-activated
target function
\begin{equation}
  y_f(\bx) = f(\boldsymbol \Gamma^\top \bx + \Delta)
  \label{eq:target_relu}
\end{equation}
where $f$ is the ReLU activation function $f(x) = \max\{x, 0 \}$.

\paragraph{Regression and Objective.}

We consider gradient descent (GD) optimization with momentum for the parameters $\boldsymbol \theta$. The update from step $t$ to $t+1$ is given as
\begin{equation}
\bm_{t+1} = \beta \bm_t + (1 - \beta) \nabla_{\boldsymbol \theta_t} \L (\boldsymbol \theta_t)
\label{eq:momentum_update}
\end{equation}
for the momentum variable $\mathbf{m}$ and
\begin{equation}
\boldsymbol\theta_{t+1} = \boldsymbol\theta_t - \eta \bm_{t+1}
\label{eq:parameter_update}
\end{equation}
for the parameters $\boldsymbol \theta$, where $\mathcal{L}$ is the loss
function, $\beta \in [0, 1)$ is the rate of momentum, and $\eta \in (0, 1)$ is
the step size.

\paragraph{Regressions Models and Parameterization.}

In Sec. \ref{sec:affine} we model the respective target function with an affine transform
\begin{equation}
\hat y(\bx) = \bw^\top \bx + b,
\label{eq:affine}
\end{equation}
and in Sec. \ref{sec:single_relu} we consider the nonlinear ReLU model
\begin{equation}\label{eq:relu}
  \hat y_f(\bx) = f(\bw^\top \bx + b).
\end{equation}
In both cases, the parameters $\boldsymbol \theta = (\bw^\top, b)^\top \in \R^{L+1}$ are weights $\bw \in \R^{L}$ and bias $b \in \R$.

We optimize the mean squared error (MSE), such that
\begin{equation}
\L (\boldsymbol \theta) = \expx{\frac{1}{2}\epsilon(\bx)^2}.
\label{eq:loss}
\end{equation}
where $\epsilon$ is the signed error given by
$\epsilon(\bx) = \hat y(\bx) - y(\bx)$ for the affine target and
$\epsilon_f(\bx) = \hat y_f(\bx) - y_f(\bx)$ for the ReLU-activated target.

To make gradient calculation easier, and interpretable, we will approximate
the gradients for the ReLU-model by replacing $y_f$ with $y$. This is still a reasonable approximation,
since for any choice of $\bx$ we have that if
\begin{equation}
  y_f(\bx) \geq \hat y(\bx) \Rightarrow \epsilon_f(\bx) = \hat y_f(\bx) - \underbrace{y_f(\bx)}_{\geq y(\bx)} = \hat y_f(\bx) - y(\bx) = \epsilon(\bx)
\end{equation}
and
\begin{equation}
  y_f(\bx) < \hat y_f(\bx) \Rightarrow \epsilon_f(\bx) = \hat y_f(\bx) - y_f(\bx) \geq \hat y_f(\bx) - y(\bx) \geq 0.
\end{equation}
That is, the error and the gradient is either identical or has the same sign
evaluated at any point $\bx$ and $\theta$.

To make calculating expected gradients of $\mathcal{L}$ easier, without
introducing any further approximations, we define a
unitary matrix $\mathbf{U}(\bw)$ (abbreviated $\mathbf{U}$) such that the vector
$\bw$ rotated by $\mathbf{U}$ is mapped to
\begin{equation}\label{eq:u_def}
\mathbf{U}\bw = (0,0, \dots ,\|\bw\|)^\top = \tilde \bw.
\end{equation}
We refer to the rotated vector as $\tilde \bw$. Again, if $\mathbf x \sim
\mathcal{N}(\mathbf{0}; \mathbf{I})$, the variables $\bx$ and $\tilde \bx$ are
identically distributed and for $\bw \neq \mathbf 0$, we also get from
\eqref{eq:u_def}
\begin{equation}
 \frac{\mathbf{U} \bw}{\|\bw\|} = (0,0,\dots ,1)^\top.
 \label{eq:u_property}
\end{equation}

\paragraph {Dying ReLU.}

A ReLU is considered dying if all inputs are negative. For inputs with infinite
support, we consider the ReLU as dying if outputs are non-zero with probability
less than some (small) $\varepsilon$
\begin{equation}\label{eq:dying_relu}
  P\left(\bw^\top \bx + b > 0\right) < \varepsilon
.
\end{equation}

% !TEX root = main.tex
\section{Regression with affine model}\label{sec:affine}

In this section, we analyze the regression of the target function $y$ from
\eqref{eq:target_special} with the affine model $\hat y$ from \eqref{eq:affine}.
From the perspective of the input and output space of this model, it is
identical to the ReLU model in \eqref{eq:relu} for all inputs that map to
positive values. On the other hand, from the parameter space perspective, we
will show that the parameter evolution is identical in certain regions of the
parameter space. The global optimum is also the same for both functions. This
allows us to re-use some of the following results later.

To study the
evolution of parameters $\boldsymbol \theta$ and momentum $\mathbf{m}$ during GD
optimization, we formulate the GD optimization as an equivalent linear autonomous system
\citep{goh2017momentum} and analyze
the behavior by inspection of the eigenvalues. For this analysis, we assume
that the inputs are distributed as $\mathbf x \sim \mathcal{N}(\mathbf{0};
\mathbf{I})$ in the training data.

% \todo{1) If there are results that translate to the other case, there could be a short comment about it at that location.}

\paragraph{Analytic gradients.}
By inserting $\hat y$ into \eqref{eq:loss}, the optimization objective can be formulated as
\begin{align}
\L ( \boldsymbol \theta)
%&=
%\expx{\frac{1}{2}((\bw - \boldsymbol \Gamma)^\top \bx + b - \Delta)^2}
%\\
&= \expx{\frac{1}{2}(\mathbf a^\top \bx + c)^2}
\end{align}
using new shorthand definitions $\ba = \bw - \boldsymbol \Gamma$ and $c = b - \Delta$. Considering that  $\boldsymbol \Gamma = (0, \dots, 0, \gamma)^\top$ from \eqref{eq:target_special}, the derivatives are
\begin{align}
\frac{\partial \L}{\partial \bw}
=
\frac{\partial \L}{\partial \ba}
=
\expx{(\ba^\top \bx + c) \bx}
%\\
= \ba \label{eq:linear_grad_w}
%\\
%&= (w_1, w_2, \dots, w_L - \gamma)^\top
\end{align}
for the weights $\mathbf{w}$ and
\begin{equation}\label{eq:linear_grad_b}
  \frac{\partial \L}{\partial b} = \frac{\partial \L}{\partial c} = c = b - \Delta.
\end{equation}
for the bias $b$.
From these results, we can see that both $\ba$ and $c$ are zero when
$\mathcal{L}(\boldsymbol \theta)$ arrives at a critical point---in this case the
global optimum.

\paragraph{Parameter evolution.}

The parameter evolution as given by Eqs. (\ref{eq:momentum_update}) and
(\ref{eq:parameter_update}) can be formulated as a linear autonomous system in
terms of $\ba$, $c$, and $\mathbf{m}$. The state $\mathbf{s}^\top =
([\mathbf{m}]_1, [\mathbf{a}]_1, \dots , [\mathbf{m}]_L, [\mathbf{a}]_L,
[\mathbf{m}]_{L+1}, c)^\top$ consists of stacked pairs of momentum and
parameters. We write the update equations \eqref{eq:momentum_update} and
\eqref{eq:parameter_update} in the form
\begin{equation}
\mathbf{s}_{t+1}
=
{\tiny
\begin{pmatrix}
\mathbf{A}
\\
& \ddots
\\
&& \mathbf{A}
\end{pmatrix}
}
\mathbf{s}_t
,
\end{equation}
where the state $\mathbf{s}$ evolves according to the constant matrix
\begin{equation}
\mathbf{A} =
    \begin{pmatrix}
      \beta & 1 - \beta \\
      -\eta \beta & 1 - \eta(1 - \beta)
    \end{pmatrix}
    ,
\end{equation}
which determines the evolution of each pair of momentum and parameter independently of other pairs.

\begin{figure}
    \centering
    \includegraphics[width=0.23\textwidth]{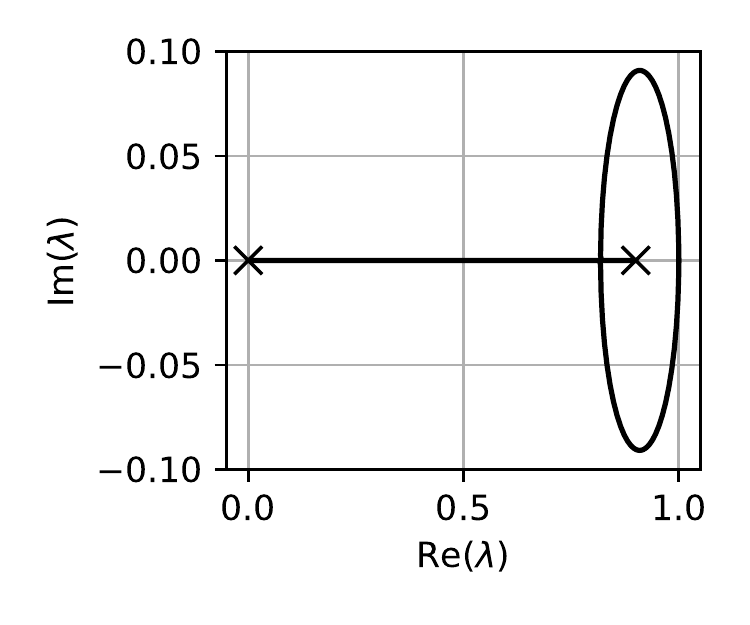}
    \includegraphics[width=0.23\textwidth]{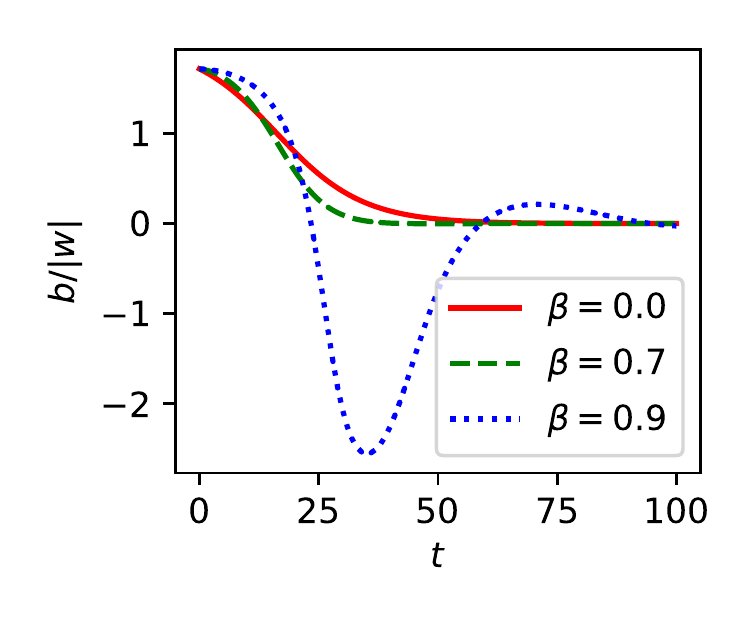}
    \caption{
Left: Eigenvalues of $\mathbf{A}$ as $\beta$ is increased from $0$ (marked as $\times$) to $1$. At $\beta\approx 0.7$ eigenvalues the become complex and take the value $1$ for $\beta = 1$.
Right: As the eigenvalues become complex, gradient descent produces oscillations.
These are three examples all originating from the same parameter coordinate.
%In the linear system, we still have convergence guarantees as long as eigenvalues have norm less than one. In the non-linear system, on the other hand, we will see that these oscillations can evolve the system into states where it can not converge to the ground truth.
}
    \label{fig:root_locus}
\end{figure}

Since the pairs evolve independently, we can study their convergence by
analyzing the eigenvalues of $\mathbf{A}$. Given that $\eta$ and $\beta$ are
from the ranges defined in Sec. \ref{sec:preliminaries}, all eigenvalues are
strictly inside the right half of the complex unit circle which guarantees
convergence to the ground truth. For step sizes $1 < \eta < 2$ eigenvalues
will still be inside the complex unit circle, still guaranteeing convergence,
but eigenvalues can be real and negative. This means that parameters will
alternate signs every gradient update, denoted ``ripples''
\cite{goh2017momentum}. Although this sign-switching can cause dying ReLU in
theory, in practice learning rates are usually $< 1$.

We plot the eigenvalues of $\mathbf{A}$ in \figref{fig:root_locus} (left) for
$\eta = 0.1$ as $\beta$ increases from $0$, i.e. GD without momentum, towards
$1$. We observe that the eigenvalues eventually become complex ($\beta \approx
0.7$) resulting in oscillation \cite{goh2017momentum} (seen on the right side).
The fraction $\frac{b}{\|\bw\|}$, as we will show in Sec. \ref{sec:single_relu},
is a good measure of to what extent the ReLU is dying (smaller means dying), and
hence we plot this quantity in particular. Note that the eigenvalues, and thus
the behavior is entirely parameterized by the learning rate and momentum, and
independent of $\gamma$. Thus, we can not adjust $\eta$ as a function of
$\gamma$ to make the system behave as the case $\gamma = 1$. We now
continue by showing how these properties translate to the case of a
ReLU activated unit.

% !TEX root = main.tex
\section{Regression with single ReLU-unit}\label{sec:single_relu}

We now want to understand the behavior of regressing to
\eqref{eq:target_relu} with the ReLU model in \eqref{eq:relu}. As
discussed in Sec. \ref{sec:preliminaries}, we will approximate the gradients
by considering the linear target function in Eq. \ref{eq:target_special}. Although
this target can not be fully recovered, the optimal solution is still the same,
and gradients share similarities, as previously discussed.
Again, we consider the evolution and convergence of parameters $\boldsymbol
\theta$ and momentum $\mathbf{m}$ during GD optimization, and assume that the
inputs are distributed as $\mathbf x \sim \mathcal{N}(\mathbf{0}; \mathbf{I})$
in the training data.

\paragraph{Similarity to affine model.}

The ReLU will output non-zero values for any $\bx$ that satisfies
$\bw^\top \bx + b > 0$. We can equivalently write this condition as
\begin{align}
\frac{\bw^\top}{\|\bw\|} \bx &> -\frac{b}{\|\bw\|}
.
\label{eq:condition}
\end{align}
We can further simplify the condition from \eqref{eq:condition} using $\mathbf{U}$ from Sec. \ref{sec:preliminaries}, which lets us consider the constraint solely in the $L$-th dimension
\begin{equation}
\footnotesize
\frac{\bw^\top}{\|\bw\|} \mathbf{U}^\top \mathbf{U} \bx
>
-\frac{b}{\|\bw\|}
\Rightarrow
\frac{\tilde \bw^\top}{\|\bw\|} \tilde \bx
>
-\frac{b}{\|\bw\|}
\Rightarrow
[\tilde {\mathbf{x}}]_L
>
-\frac{b}{\|\bw\|}.
\end{equation}
From our assumption about the distribution of training inputs follows that $[\tilde {\mathbf{x}}]_L  \sim \mathcal{N} (0, 1)$.

With this result, we can compute the probability of a non-zero output from the ReLU as
\begin{equation}
\footnotesize
P([\tilde {\mathbf{x}}]_L  > -\frac{b}{\|\bw\|})
=
1 - P\left([\tilde {\mathbf{x}}]_L < -\frac{b}{\|\bw\|}\right)
=
\Phi\left(\frac{b}{\|\bw\|}\right)
,
\end{equation}
where $\Phi$ is the cumulative distribution function (CDF) of the standard normal distribution.

Using \eqref{eq:dying_relu}, we see that dead ReLU is equivalent of
$\Phi\left(\frac{b}{\|\bw\|}\right) < \varepsilon$. This is equivalent of
$\frac{b}{\|\bw\|} < -\Phi^{-1}(\varepsilon)$ which defines a ``dead'' cone in parameter
space. We can also
formulate a corresponding ``linear'' cone. In this case we get
$\frac{b}{\|\bw\|} > \Phi^{-1}(\varepsilon)$ which is the same cone mirrored
along the $b$-axis. In the linear cone, because of the similarity to the affine
model, we know that parameters evolve as described in Sec. \ref{sec:affine} with
increased oscillations as momentum is used. We will now investigate the
analytical gradients to see how these properties translate into that
perspective.

%\todo{It is not clear how this result is used in the next step. A comment is needed here.}

\paragraph{Analytic gradients.}
By inserting $\hat y_f$ into \eqref{eq:loss}, the optimization objective can be formulated as
\begin{align}\label{eq:loss_relu}
\L ( \boldsymbol \theta)
&=
\expx{\frac{1}{2}\mathbbm{1}_{\bw^\top \bx + b > 0}(\mathbf a^\top \bx + c)^2}
,
\end{align}
where we use the indicator function $\mathbbm{1}_{\bw^\top \bx + b > 0}$ to model the ReLU activation.
The derivatives are
\begin{equation}\label{relu_grad_exp_w}
  \frac{\partial \L}{\partial \bw} = \expx{\mathbbm{1}_{\bw^\top \bx + b > 0}(\ba^\top \bx + c)\bx}
\end{equation}
for the weights and
\begin{equation}\label{relu_grad_exp_b}
  \frac{\partial \L}{\partial b} = \expx{\mathbbm{1}_{\bw^\top \bx + b > 0}(\ba^\top \bx + c)}
\end{equation}
for the bias.
As in Sec. \ref{sec:affine}, the optimal fit %of the ReLU from \eqref{eq:relu} for the target function from \eqref{eq:target_special}
is given by weights $\bw = \boldsymbol \Gamma$ and bias $b = \Delta$. These parameters give $\ba = \mathbf 0$ and $c = 0$ and set the gradients in \eqref{relu_grad_exp_w} and \eqref{relu_grad_exp_b} to zero.

% \begin{figure}
%     \centering
%     \includegraphics[width=0.27\textwidth]{figures/output_pdf}
%     \includegraphics[width=0.20\textwidth]{figures/cones_small}
%     \caption{
% Left: Dying ReLU parameters form a cone (white).
% On its surface, non-zero
% outputs have probability $\varepsilon$, in this case $10^{-15}$,
% and even smaller inside. Where the
% probability is close to $1$ (yellow), the ReLU is approximately an affine
% transformation and where the value is close to $0$ (blue), the ReLU is instead
% approximately constant $0$. Right: Cones of parameters
% corresponding to dying ReLU (blue) and affine transformations (orange).
% }
%
%     \label{fig:output_pdf}
% \end{figure}

By changing base using $\mathbf{U}$, we can compute the derivatives,
\begin{equation}
\label{relu_gradient_i_not_l}
\left[\mathbf{U} \frac{\partial \mathcal L}{\partial \mathbf{w}}\right]_{i \neq L}
=
[\tilde {\mathbf{a}}]_{i \neq L}  \Phi\left(\frac{b}{\| \mathbf{w}\|}\right)
\end{equation}
for the weights in dimensions $i \neq L$, and
\begin{equation}
\left[\mathbf{U} \frac{\partial \mathcal L}{\partial \mathbf{w}}\right]_{L}
=
[\tilde {\mathbf{a}}]_L \Phi\left(\frac{b}{\| \mathbf{w}\|}\right) +
\left(c -  [\tilde {\mathbf{a}}]_L \frac{b}{\|\mathbf{w}\|}\right)\phi\left(\frac{b}{\|\mathbf{w}\|}\right)
           \label{relu_gradient_l}.
\end{equation}
for dimension $L$, where we used $\phi$ as density function of the standard normal distribution.
For the bias $b$ we have
\begin{align}
  \frac{\partial \mathcal L}{\partial b} = [\tilde {\mathbf{a}}]_L \phi\left(\frac{b}{\| \mathbf{w}\|}\right)+ c \ \Phi\left(\frac{b}{\|\mathbf{w}\|}\right).
\end{align}
Full derivations are listed in Appendix \ref{app:relu_gradients}.

\paragraph{Critical points in parameter space.}

With the gradients from above, we can analyze the possible parameters (i.e.
optima and saddle points) to which GD can converge when the derivatives are
zero. First of all, \emph{global optimum} is easily verified to have gradient zero when
$\tilde \ba = \ba = \mathbf{0}$ and $c = 0$.

\emph{Saddle points} correspond to dead ReLU and occur when $\bwfrac
\rightarrow -\infty$ since then $\Phi(\bwfrac) = 0$, $\phi(\bwfrac) = 0$, and
$\bwfrac \phi(\bwfrac) \rightarrow 0$. This equals the case that $\bw = \mathbf{0}$
for any $b < 0$ which can be verified by plugging these values into
\eqref{eq:loss_relu}. These saddle points form the center of the dead cone.
Note in practice that these limits in practice
occur already at, for
example $\bwfrac = \pm 4$, since then
$\phi(\pm 4) \approx 10^{-4}$. The implication of this is that an entire dead
cone can be considered as saddle points in practice, and that parameters will
converge on the surface of the cone rather than in its interior.

For $b > 0$, we instead have $\bwfrac \rightarrow \infty$ and thus
$\phi(\bwfrac) = 0$ and $\Phi(\bwfrac) = 1$ and the gradients can be verified to
equal those of the affine model in Sec. \ref{sec:affine}, as expected. This
verifies that parameter evolution in the linear cone will be approximately identical
to the affine model.

\paragraph{Simplification to enable visualization.}
To continue our investigation of the parameter evolution and, in particular,
focus on how the target variance $\gamma$ and momentum evolve parameters into
the dead cone we will make some simplifications. For this, we will assume that
$[\tilde {\mathbf{a}}]_{i \neq L} = 0$ which enables us to express gradients
without $\mathbf U$ as
\begin{equation}
  \frac{\partial \mathcal L}{\partial [\bw]_{i \neq L} } = 0
.
\end{equation}
For the weight $[\bw]_L$ we can prove, see Appendix \ref{app:single_relu_criticals}, that
\begin{equation}
\footnotesize
\frac{\partial \mathcal L}{\partial [\bw]_L}
=
[\mathbf{a}]_L \Phi \left(\frac{b}{\|\bw\|} \right) + \rho
\left(c - \rho [\mathbf{a}]_L \frac{b}{\|\bw\|}\right) \phi\left(\frac{b}{\|\bw\|}\right)
,
\end{equation}
where $\rho = \mathrm{sign}( [\bw]_L)$ and $[\mathbf a]_L = [\bw]_L - \gamma$.
For the bias $b$ we get
\begin{equation}
\frac{\partial \mathcal L}{\partial b} = \rho [\mathbf a]_L \phi\left(\frac{b}{\|\bw\|}\right)+ c \ \Phi\left(\frac{b}{\|\bw\|}\right).
\end{equation}
This means, if all $ [ \tilde {\mathbf a}]_i = 0$ for $i \neq L$, then only the
weight $[\bw]_L$ and bias $b$ evolve. We can now plot the vector fields in
these parameters to see how they change w.r.t. $\gamma$.

\paragraph{Influence of $\gamma$ on convergence without momentum.}

\begin{figure}
    \centering
    \includegraphics[width=0.8\textwidth]{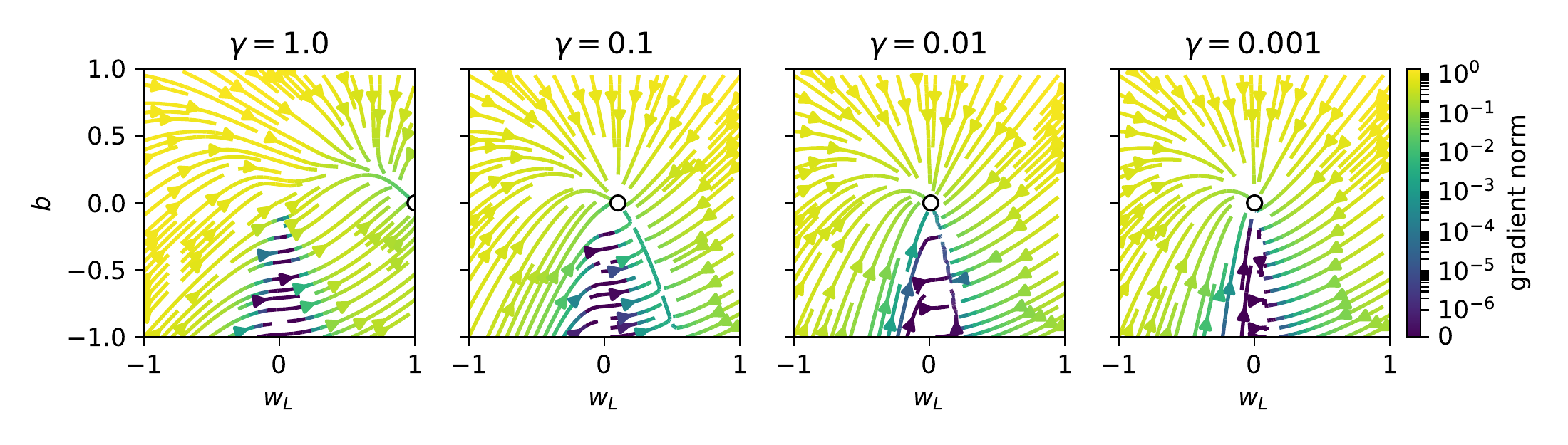}
    \caption{
Parameter evolution from GD. The white dot marks the global optimum. %Note the blue region and that the color scale is logarithmic.
Approaching $[\bw]_L = 0$ in the lower half, gradient norm becomes very small, GD gets stuck , and the ReLU dies.}
    \label{fig:streamplots_small_gamma}
\end{figure}

The first key take-away when decreasing $\gamma$ is that the global optimum will
be closer to the origin and eventually between the dead
and the linear cone. This location of the optimum is particularly sensitive, since
in this case the parameters in the linear cone evolve towards the dead cone, and in
addition exhibits oscillatory behavior for large $\beta$.
The color scheme in \figref{fig:streamplots_small_gamma} verifies that, like the
probability of non-zero outputs in the dead cone, the gradients also
tend to zero there. In the lower right quadrant, we can see an
attracting line that is shifted towards the dead cone as $\gamma$ decreases,
eventually ending up inside the dead cone. For this case, when $\gamma = 0.001$,
we can follow the lines to see that most parameters originating from the right
side end up in the dead cone. Parameters originating in and near the linear cone
approach the ground truth, and the lower left quadrant evolves first towards the
linear cone before evolving towards the ground truth.

When adding momentum, remember that parameter evolution in and near the linear cone exhibits oscillatory
behavior. The hypothesis at this moment is that parameters originating from the
linear cone can oscillate over into the dead cone and get stuck there. For the
other regions they either evolve as before into the dead cone, or first into
the linear cone and then into the dead cone by oscillation. We will now evaluate and visualize
this.

%Also, we can by inspection of this vector field get a rough understanding what happens if momentum is introduced. To know exactly which regions converge where, and how momentum changes the outcome, we performed additional experiments as explained in the following section.

\paragraph{Regions that converge to global optimum.}

We are interested in distinguishing the regions from which parameters will converge to the
global optimum and dead cone respectively. For this, we apply the update rules,
Eqs. (\ref{eq:momentum_update}) and (\ref{eq:parameter_update}), with $\eta = 0.1$, until updates
are smaller (in norm)
than $10^{-6}$.

\begin{figure}[ht!]
    \centering
    \includegraphics[width=0.6\textwidth]{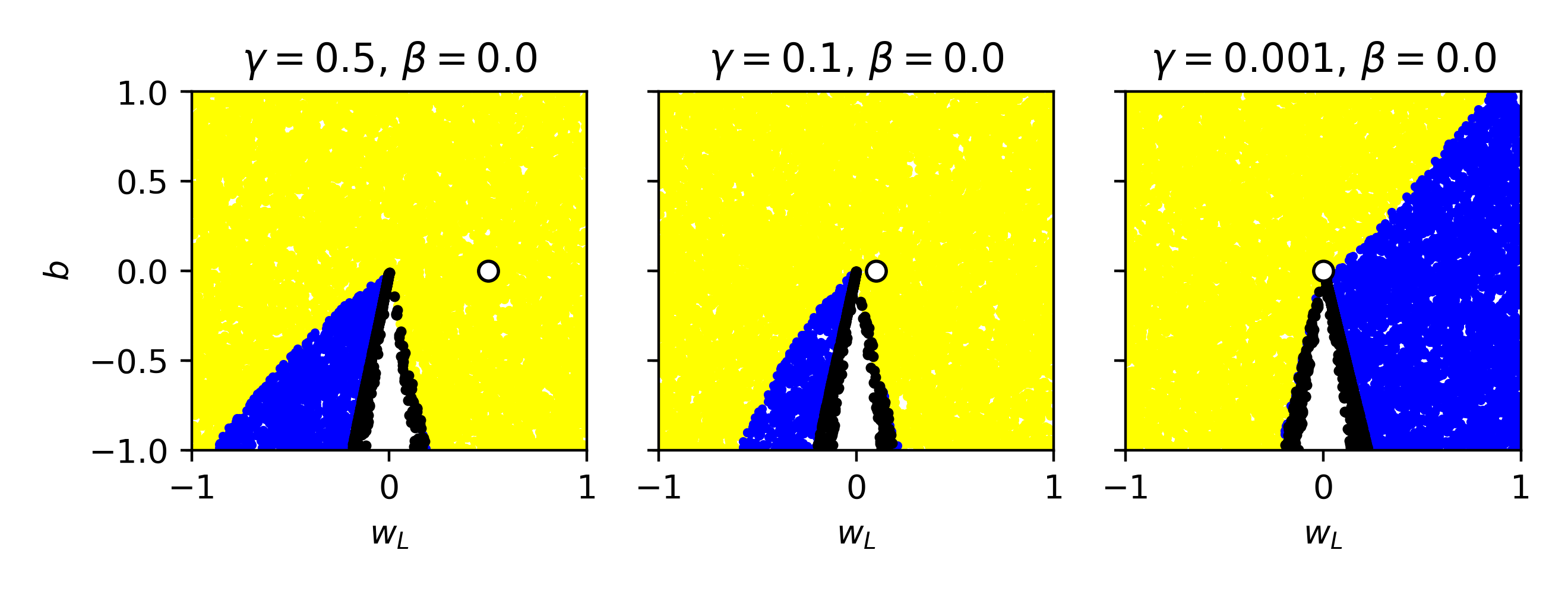}
    \caption{
Parameter evolution from GD without momentum. Global optimum is shown as a white circle.
The yellow dots correspond to initial parameter coordinates that converge at the
ground truth. The blue dots all converge in the dead cone (where they stop are depicted as black dots).
As can be seen, small
target variance $\gamma$ is not alone sufficient to lead to a majority
of dying ReLU.
}
    \label{fig:converging_non_beta}
\end{figure}
\figref{fig:converging_non_beta} shows the results for GD without momentum ($\beta =
0$). We see that the region that converges to the dead cone changes with
$\gamma$, and eventually switches sign, when $\gamma$ becomes small. The
majority of initializations still converges at the ground truth.
\begin{figure}[ht!]
    \centering
    \includegraphics[width=0.8\textwidth]{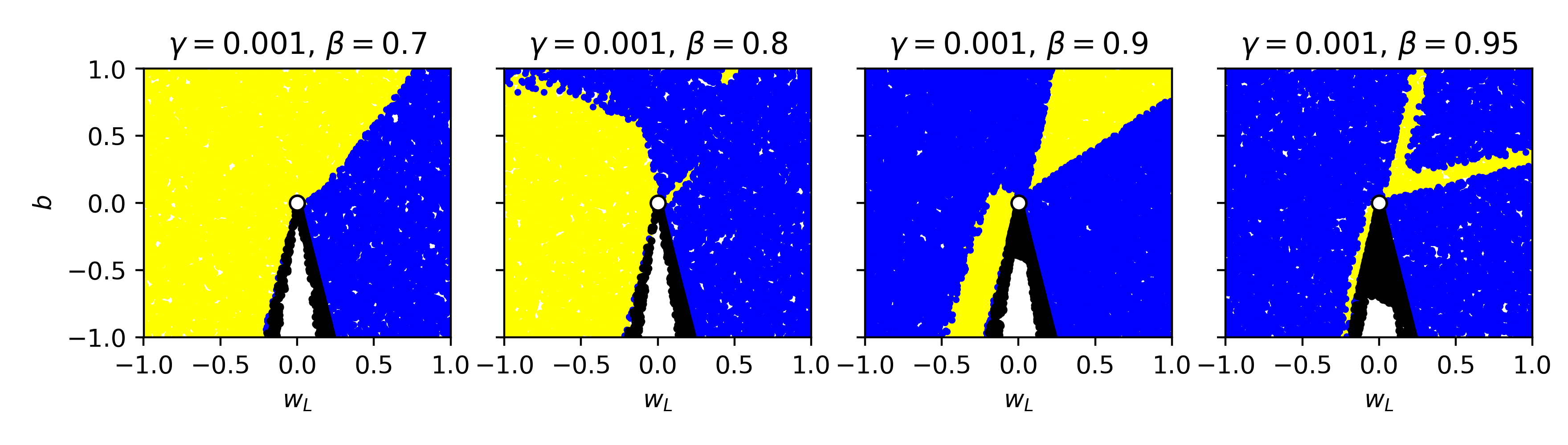}
    \caption{
Parameter evolution from GD without momentum. Colors as in \figref{fig:converging_non_beta}. Momentum and small scale $\gamma = 0.001$ increasingly leads to dying ReLU. The black regions corresponding to converged parameters are growing in the interior of the dead cone.}
    \label{fig:converging_beta}
\end{figure}

\figref{fig:converging_beta} shows the results with momentum. The linear
autonomous system in Sec. \ref{sec:affine} with $\eta= 0.1$ has complex
eigenvalues for $\beta > 0.7$ which lead to oscillation. This property
approximately translates to the ReLU in the linear cone, where we expect
oscillations for $\beta > 0.7$. Indeed, we observe the same results as without
momentum for $\beta \leq 0.7$, but worse results for larger $\beta$. Eventually,
only small regions of initializations are able to converge to the global
optimum.

%Knowing the form and behavior of the linear system, we now instead consider
%learning the parameters of
%%
%\begin{equation}\label{eq:single_relu}
%  \hat y(\bx) = \sigma(\bw^\top \bx + b).
%\end{equation}
%
%\todo{Is this the same equation as \eqref{eq:relu}?}
%We will see that this behaves exactly as the linear model as
%$\frac{b}{\|w\|} \rightarrow \infty$ \todo{What does behaves mean? Is that referring to the learning / regression?}. Conversely, it will be
%equal to zero as $\frac{b}{\|w\|} \rightarrow -\infty$ \todo{What will be equal to zero? The regression result? The eigenvalues?}.
%\todo{steps missing}

\section{Deeper architectures}

In this section we will address two questions. (1) Does the problem persist in deeper
architectures, including residual networks with batch normalization?
(2) Since the ReLU is a linear operator for positive coefficients,
can we simply scale the weight initialization
and learning rate and have an equivalent network and learning procedure? For the
latter we will show that it is not possible for deeper architectures.

\paragraph{Relevance for models used in practice.}

We performed regression on the same datasets as in Sec. \ref{sec:introduction},
with $\gamma = 10^{-4}$. We confirm the findings of \citet{lu2019dying} that ReLUs
are initialized dead in deeper ``vanilla'' NNs, but not in residual networks due
to the batch normalization layer that precedes the ReLU. Results are shown in
\figref{fig:deeper_architectures}. We further find that ReLUs die more often,
and faster, the deeper the architecture, even for residual networks with
batch normalization. We can also conclude that, in these experiments, stochastic
gradient descent does not produce more dead ReLU during optimization.

\begin{figure}[ht!]
    \centering
    \includegraphics[width=0.8\textwidth]{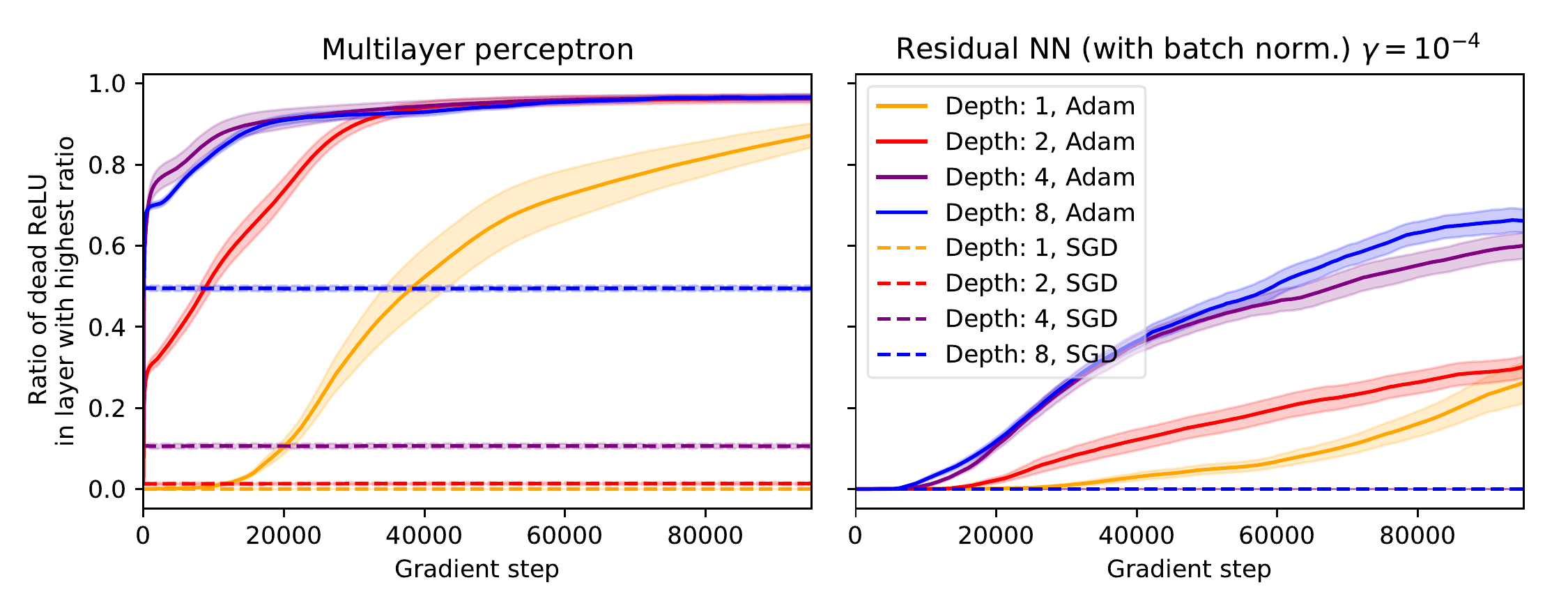}
    \caption{Average results over four datasets showing that dead ReLU is
             aggravated by deeper architectures, even for residual networks
             with batch normalization before the ReLU activation.
             Horizontal axis shows the proportion of dead ReLU in the layer
             with the highest ratio of dead ReLU. In the multi-layer
             perceptron case, if one layer has only dead ReLUs, then
             gradients are zero for all parameters except the last biases, no
             matter which layer ``died''.}
    \label{fig:deeper_architectures}
\end{figure}

\paragraph{Parameter re-scaling.}

For $\gamma > 0$, the ReLU function has the property $f(\gamma \bx) = \gamma
f(\bx)$ and thus $\gamma \hat y(\bx) = f(\gamma \bw^\top \bx + \gamma b)$. By
rescaling the weight and bias initializations (not learning rate) by $\gamma$,
the parameter trajectories during learning will proportional no matter the
choice of $\gamma > 0$. That is, ReLUs will die independent of $\gamma$. This is
a special case though, since for any architecture $\hat y$ with one or more
hidden layers, it is not possible to multiply parameters by a single scalar
$\nu$ such that the function is identical to $\gamma \hat y$. Proof is provided
in Appendix \ref{app:parameter_scaling}. Also, we still
have a problem if we do not know $\gamma$ in advance, which holds for example in
the reinforcement learning setting.

% !TEX root = main.tex
\section{Conclusion and future work}

Target normalization is indeed a well-motivated and used practice, although we
believe a theoretical understanding of \emph{why} is both important and lacking.
We take a first stab at the problem by understanding the properties in the
smallest moving part of a neural network, a single ReLU-activated affine
transformation. Gradient descent on parameters of an affine transformation can
be expressed as a discrete-time linear system, for which we have tools to
explain the behavior. We provide a geometrical understanding of how these
properties translate to the ReLU case, and when it is considered dead. We
further illustrated that weight initialization from large regions in parameter
space lead to dying ReLUs as target variance is decreased along with momentum
optimization, and show that the problem is still relevant, and even aggravated,
for deeper architectures. Remaining questions to be answered include how we
can extend the analysis for the single ReLU to the full network. Here, the
implicit target function of the single ReLU is likely neither of the linear
or piecewise linear functions, and inputs are not Gaussian distributed and vary
over time.

\bibliography{references}
\bibliographystyle{unsrtnat}

% !TEX root = main.tex
\section*{Acknowledgements}
%This work was funded by the Swedish Fund for Strategic
%Research (SSF) through the project Factories of the Future (FACT).
This work was funded by the ***********
through the project ************.

% !TEX root = main.tex
\appendix

\section{Regression experiment details}\label{regression_details}
NN used for regression was of the form
\begin{equation}
  \hat y(\bx) = \bw^\top \sigma (W\bx + b^{(1)}) + b^{(2)}
\end{equation}
where $\bw \in \R^{200 \times 1}$ and $W \in \R^{200 \times n}$ where
$n$ is the dimensionality of the input data. All elements in
$W$ and $b^{(1)}$ were initialized from $U(-\frac{1}{\sqrt{n}},\frac{1}{\sqrt{n}})$.
The elements of $\bw$ and $b^{(2)}$ were initialized from
$U(-\frac{1}{\sqrt{200}},\frac{1}{\sqrt{200}})$. Batch size was $64$ and the Adam
optimizer was set to use the same parameters as suggested by \cite{adam}.
Every experiment was run $8$ times with different seeds, effecting the random
ordering of mini-batches and initialization of parameters. Each experiment
was allowed $250000$ gradient updates before evaluation. Evaluation was performed
on a subset of the \emph{training set}, as we wanted to investigate the fit to the
data rather than generalization.

\section{Single ReLU gradients}\label{app:relu_gradients}

For brevity we will use $\mathbbm 1$ in place of $\mathbbm{1}_{\tilde x_L >
-\frac{b}{\|\bw\|}}$ below. Subscript notation on non-bold symbols
here represents single dimensions of the vector counterpart.
The gradient w.r.t. $\bw$ is
\begin{align*}
  \frac{\partial \L}{\partial \bw} &= \expx{\mathbbm{1}_{\bw^\top \bx + b > 0}(\ba^\top \bx + c)\bx} \\
                                   &= \expx{\mathbbm{1}_{\bw^\top U^\top U \bx + b > 0}(\ba^\top U^\top U\bx + c)U^\top U\bx} \\
                                   &= U^\top \expxt{\mathbbm{1}(\tilde \ba^\top \tilde \bx + c)\tilde \bx}
\end{align*}
Multiplying by $U$ from the left and looking at a dimension $i \neq L$, we get:
\begin{align*}
  \left[U\frac{\partial \L}{\partial \bw}\right]_i &=&& \expxt{\mathbbm{1}(\tilde \ba^\top \tilde \bx + c)\tilde x_i} \\
       &=&& \expxt{\mathbbm{1}\left(\tilde x_i\left(\sum_{j\neq i}^L \tilde a_j \tilde x_j  + c\right) + \tilde a_i \tilde x_i^2\right)} \\
       &=&& \underbrace{\expxt{\mathbbm{1}\tilde x_i}}_{=0}\expxt{\mathbbm{1}\left(\sum_{j\neq i}^L \tilde a_j \tilde x_j + c\right)} \\
       & && +\tilde a_i\expxt{\mathbbm{1} \tilde x_i^2 } \\
       &=&& \tilde a_i\underbrace{\text{Var}(\tilde x_i)}_{=1}\mathbb E_{\tilde x_L} \left[\mathbbm{1} \right] \\
       &=&& \tilde a_i \Phi\left(\frac{b}{\|\bw\|}\right) \\
\end{align*}
For $i=L$ we instead have
\begin{align}
  \left[U\frac{\partial \L}{\partial \bw}\right]_L &=&& \expxt{\mathbbm{1}(\tilde \ba^\top \tilde \bx + c)\tilde x_L} \nonumber\\
       &=&& \expxt{\mathbbm{1}\left(\tilde x_L \sum_{j=1}^{L-1} \tilde a_j \tilde x_j + \tilde a_L \tilde x_L^2 + c\tilde x_L\right)} \nonumber\\
       &=&& \expxt{\mathbbm 1 \tilde x_L} \left(\sum_{j=1}^{L-1} \tilde a_j\underbrace{\expxt{\mathbbm 1 \tilde x_j}}_{=0}\right) \nonumber\\
       & && +\tilde a_L \expxt{\mathbbm 1 \tilde x_L^2} \nonumber \\
       & && +c \expxt{\mathbbm 1 \tilde x_L} \label{eq:gradient_l_not_done}
\end{align}
We can calculate the second term with integration by parts
\begin{align*}
    \expxt{\mathbbm 1 \tilde x_L^2} &=&& \frac{1}{\sqrt{2\pi}}\int\limits_{-\frac{b}{\|\bw\|}}^\infty \tilde x_L \cdot \tilde x_L e^{-\frac{\tilde x_L^2}{2}} d\tilde x_L \\
        &=&& \frac{1}{\sqrt{2\pi}}\left[-\tilde x_L e^{-\frac{\tilde x_L^2}{2}}\right]_{-\frac{b}{\|\bw\|}}^{\infty} \\
        & && +\frac{1}{\sqrt{2\pi}}\int\limits_{-\frac{b}{\|\bw\|}}^\infty e^{-\frac{\tilde x_L^2}{2}} d\tilde x_L \\
        &=&& -\frac{b}{\|\bw\|}\phi\left(\frac{b}{\|\bw\|}\right) + \Phi\left(\frac{b}{\|\bw\|}\right)
\end{align*}
The third term:
\begin{align*}
    \int\limits_{-\frac{b}{\|\bw\|}}^\infty \frac{\tilde x_L}{\sqrt{2\pi}} e^{-\frac{\tilde x_L^2}{2}}
           &= -\left[\frac{1}{\sqrt{2\pi}}e^{-\frac{\tilde x_L^2}{2}}\right]_{-\frac{b}{\|\bw\|}}^{\infty} \\
           &= -\left[\phi(\tilde x_L)\right]_{-\frac{b}{\|\bw\|}}^{\infty} \\
           &= \phi\left(\frac{b}{\|\bw\|}\right).
\end{align*}
By substituting these expressions back in to \eqref{eq:gradient_l_not_done} and simplifying we get \eqref{relu_gradient_l}.
The derivation of gradients w.r.t. the bias is very similar to the gradients w.r.t. to $\bw$
and are omitted here for brevity.

\section{Gradients after convergence in the $(L-1)$ first dimensions}\label{app:single_relu_criticals}

If we assume the $L-1$ first dimensions of $\tilde a$ are zero, we
can show that the $L-1$ first dimensions of $\bw$ all have gradient zero.
We can then also express the gradient without the unitary matrix $U$ that
otherwise needs to be calculated for every new $\bw$. We present this as a
proposition by the end of this section. First, we need to show some preliminary
steps.

\begin{lemma}\label{lemma:u_gamma_subspace}
  If
  $$
    \tilde \ba = (0, ..., 0, \tilde a_L)^\top
  $$
  then
  $$
    U\Gamma = (0, ..., 0, k \gamma)^\top
  $$
  where $k$ is either $-1$ or $1$.
\end{lemma}
\begin{proof}
  As defined in \eqref{eq:u_def}, we have $U\bw = (0,...,0,\|\bw\|)^\top$. Since
  \begin{align*}
    \tilde \ba &= U \ba \\
               &= U\bw - U\Gamma \\
               &= (0 - [U\Gamma]_1,...,0 - [U\Gamma]_{L-1}, \|\bw\| - [U\Gamma]_L)
  \end{align*}
  we must have
  $U\Gamma = (0,...,0,\ell)^\top$.
  Since U is unitary we must have $|\ell| = \gamma$ which is solved by
  $\ell = \pm\gamma$.
\end{proof}
\begin{lemma}\label{lemma:weight_and_u}
  If
  $$
    \tilde \ba = (0, ..., 0, \tilde a_L)^\top
  $$
  then
  $$
    \bw = (0, ..., 0, w_L)^\top
  $$
\end{lemma}
\begin{proof}
  We have from Lemma \ref{lemma:u_gamma_subspace} that $\Gamma$ and $U\Gamma$
  lie in the same one-dimensional linear subspace
  $\mathcal K$ spanned by $(0,...,1)^\top$.
  Since $U$ is unitary and $U\bw$ also is in $\mathcal K$, then so must $\bw$.
  This implies that $\bw = (0,0,...,0,w_L)^\top$.
\end{proof}
\begin{lemma}\label{lemma:sign_wl}
  If
  $$
    \tilde \ba = (0, ..., 0, \tilde a_L)^\top
  $$
  then for any vector $\mathbf v$ in the linear subspace $\mathcal{K}$ spanned
  by $(0,0,...,1)^\top$ we have
  $$
    U\mathbf v = U^\top \mathbf v = \rho \mathbf v
  $$
  where $\rho = \mathrm{sign}(w_L)$.
\end{lemma}
\begin{proof}
From Lemma \ref{lemma:weight_and_u}, we have $\bw = (0,...,0,w_L)^\top$,
which implies
$\frac{\bw}{\|\bw\|} = (0,...,0,\rho)^\top$.
By the property of $U$ as defined in \eqref{eq:u_def} we have
$$
  U\frac{\bw}{\|\bw\|} = (0,...,0,1)^\top.
$$
Multiplying by $U^\top$ from the left, we have
\begin{align*}
  U^\top U \frac{\bw}{\|\bw\|} &= U^\top (0,...,0,1)^\top \\
               &= (0,...,0,\rho)^\top
\end{align*}
This implies that the last row and last column of $U$ is
$(0,...,0,\rho)$.
Then
\begin{align*}
  U \mathbf v &= U (0,...,0,v_L)^\top \\
              &= (0,...,0,\rho v_L)^\top \\
              &= \rho \mathbf v
\end{align*}
Since the last row and last column are the same for $U$, the above also holds
for the transpose $U^\top$.
\end{proof}
\begin{proposition}
  If
  $$
    \tilde \ba = (0, ..., 0, \tilde a_L)^\top
  $$
  then
  $$
    \frac{\partial \mathcal L}{\partial w_L} =
      a_L \Phi(\frac{b}{\|w\|})+ \rho(c -  \rho a_L \frac{b}{\|w\|})\phi(\frac{b}{\|w\|})
  $$
  and
  $$
    \frac{\partial \mathcal L}{\partial w_i} =
      0
  $$
  for $i \neq L$ and
  where $\rho = \mathrm{sign}(w_L)$ and $a_L = w_L - \gamma$.
  The gradient w.r.t. the bias is
  $$
    \frac{\partial \mathcal L}{\partial b} = \rho a_L \phi(\frac{b}{\|w\|})+ c \Phi(\frac{b}{\|w\|}).
  $$
\end{proposition}
\begin{proof}
    From Lemma \ref{lemma:u_gamma_subspace} and \ref{lemma:sign_wl} we have that
    $$
      U\ba = U(0, 0,...,0,w_L - \gamma)^\top = (0,0,...,0,\rho (w_L - \gamma))^\top
    $$
    and therefore $\tilde a_L = \rho a_L$. All occurrences of $\tilde a_L$ can now be
    replaced with this quantity.
    The rotated gradient is
    $$
    U \frac{\partial \mathcal L}{\partial w} = (0, 0, ..., 0, \xi)^\top \in \mathcal K
    $$
    where $\xi$ is given by \eqref{relu_gradient_l}. Hence
    $\frac{\partial \mathcal L}{\partial w} = U^\top \left(U \frac{\partial \mathcal L}{\partial w}\right) = (0,0,...,0,\rho\xi)^\top$
    by Lemma \ref{lemma:sign_wl}.
\end{proof}

\section{Scaling of parameters}
\label{app:parameter_scaling}

Before stating the problem, we define an N-hidden layer neural network
recursively
\begin{equation}
  \hat y(\bx) = \bw^\top y^{(N)}(\bx) + b
\end{equation}
where
\begin{equation}
  y^{(n)}(\bx) = f(W^{(n)} y^{(n-1)} + \bb^{(n)})
\end{equation}
and
\begin{equation}
  y^{(0)}(\bx) = \bx.
\end{equation}

Denote the joint collection of parameters $\bw$, $b$, $\{W^{(n)}\}_{i=1}^N$ and
$\{\bb^{(n)}\}_{i=1}^N$ as $\theta$. We denote $\hat y(\bx) = \hat y(\bx, \theta)$
to make explicit the dependence on $\theta$. We now investigate whether
a scaling of $\hat y$ can be achieved by scaling all parameters by a single
number.

\begin{theorem}
  Given some $\gamma > 0$, $\gamma \neq 1$, there exists no $\nu > 0$ such that
  \begin{equation}\label{eq:mlp_nu}
    \gamma \hat y(\bx, \theta) = \hat y(\bx, \nu \theta)
  \end{equation}
\end{theorem}

\begin{proof}
  First, denote $\gamma = \gamma_{N+1}$ and factorize into $N+1$ factors $\alpha_i > 0$
  \begin{equation}
    \gamma_n = \prod\limits_{i = 1}^n \alpha_i.
  \end{equation}
  Multiplying $\hat y$ by $\gamma$ gives:
  \begin{align}
    \gamma \hat y(\bx) &= \gamma_{N+1} \bw^\top y^{(N)}(\bx) + \gamma_{N+1} b \\
                       &= \frac{\gamma_{N+1}}{\gamma_{N}} \bw^\top \gamma_N y^{(N)}(\bx) + \gamma_{N+1} b \\
                       &= \alpha_{N+1} \bw^\top \gamma_N y^{(N)}(\bx) + \gamma_{N+1} b \label{eq:mlp_nu_ineq1}
  \end{align}
  and similarly
  \begin{align}
    \gamma_n y^{(n)}(\bx) &= \gamma_n f(W^{(n)}y^{(n-1)}(\bx) + \bb^{(n)}) \\
                          &= f(\alpha_n W^{(n)} \gamma_{n-1} y^{(n-1)}(\bx) + \gamma_n \bb^{(n)}). \label{eq:mlp_nu_ineq2}
  \end{align}
  We finally define $\gamma_0 y^{(0)}(\bx) = 1 \cdot \bx$.

  If a $\nu$ exists that satisfies Eq. (\ref{eq:mlp_nu}), then we must have all
  $\alpha_i = \nu$.
  If $\gamma > 1$, then $\gamma_n = \nu^n > \nu = \alpha_n$
  and thus in Eqs. (\ref{eq:mlp_nu_ineq1}) and (\ref{eq:mlp_nu_ineq2}) the weights
  and biases are not multiplied by the same number $\gamma_n \neq \alpha_n$.
  Similarly, this holds for $\gamma < 1$.

\end{proof}

\end{document}